\documentclass[lettersize,journal]{IEEEtran}
\usepackage{amsmath,amsfonts}
\usepackage{algorithmic}
\usepackage{algorithm}
\usepackage{array}
\usepackage[caption=false,font=footnotesize,labelfont=sf,textfont=sf]{subfig}
\usepackage{textcomp}
\usepackage{stfloats}
\usepackage{url}
\usepackage{verbatim}
\usepackage{graphicx}
\usepackage{cite}

\usepackage{amssymb}
\usepackage{float}
\usepackage{multirow}
\usepackage{epsfig}
\usepackage{epstopdf}
\usepackage{makecell}
\usepackage{bm}
\usepackage{times}
\usepackage{booktabs}
\usepackage{microtype}
\usepackage{hyperref}
\usepackage{color}
\usepackage{diagbox}
\usepackage[english]{babel}
\usepackage{amsthm}

\theoremstyle{plain}
\newtheorem{proposition}{Proposition}
\begin{document}

\title{
Unified Graph Prompt Learning via Low-Rank Graph Message Prompting

}

\author{Beibei Wang, Bo Jiang*\thanks{* Corresponding author (e-mail: jiangbo@ahu.edu.cn).}, Ziyan Zhang, Jin Tang 
	\IEEEcompsocitemizethanks{\IEEEcompsocthanksitem{Beibei Wang, Ziyan Zhang and Jin Tang are with the School of Computer Science and Technology,  Anhui University, Hefei 230601, China.}
		\IEEEcompsocthanksitem{Bo Jiang is with
		State Key Laboratory of Opto-Electronic Information Acquisition and Protection Technology, School of Computer Science and Technology, Anhui University, Hefei 230601, China.
		}
	}
}

\markboth{Journal of \LaTeX\ Class Files,~Vol.~14, No.~8, August~2021}%
{Shell \MakeLowercase{\textit{et al.}}: A Sample Article Using IEEEtran.cls for IEEE Journals}

\IEEEpubid{0000--0000/00\$00.00~\copyright~2021 IEEE}

\maketitle

\begin{abstract} 
Graph Data Prompt (GDP), which introduces specific prompts in graph data for efficiently adapting pre-trained GNNs, has become a mainstream approach to graph fine-tuning learning problem. 
However, existing GDPs have been respectively designed  for distinct graph component (e.g., node features, edge features, edge weights) and thus operate within limited prompt spaces for
graph data. To the best of our knowledge, it still lacks a unified  prompter suitable for targeting all graph components simultaneously. 
To address this challenge,  in this paper, we 
first propose to reinterpret a wide range of existing GDPs from an aspect of Graph Message Prompt (GMP) paradigm. 
Based on GMP, we then introduce a novel graph prompt learning approach, termed Low-Rank GMP (LR-GMP), which leverages low-rank prompt representation to achieve an effective 
and compact graph prompt learning. 
Unlike traditional GDPs that target distinct graph components separately, LR-GMP concurrently performs prompting on all graph components in a unified manner, thereby achieving significantly superior generalization and robustness on diverse downstream tasks. 
Extensive experiments on several graph benchmark datasets demonstrate the effectiveness and advantages of our proposed LR-GMP.

\end{abstract}


\section{Introduction}
%
%
\IEEEPARstart{R}{ecently}, the ``pre-training, fine-tuning'' method has been widely adopted in GNN-based graph learning tasks~\cite{liu2025graph,adaptergnn24,G-adapter24,yang2025graphlora,gppt22,gpf24,EdgePrompt,graphtop25}.  It first trains GNN models on the pre-training dataset in a self-supervised manner and then adapts the pre-trained GNN models for various downstream tasks.  

Inspired by the great success of prompt tuning methods in Natural Language Processing~\cite{liu2023pre} and Computer Vision~\cite{xiao2025prompt},
graph prompt learning has emerged as a promising and parameter-efficient adaption paradigm~\cite{gpt_survey,gpt_survey2}. 
A mainstream approach is to design various Graph Data Prompts (GDPs) that introduce task-specific prompts into graph data based on node features~\cite{graphprompt23,gpf24,vgp}, edge features~\cite{EdgePrompt}, and edge weights~\cite{AAGOD,uniprompt,graphtop25}, as illustrated in Fig.~\ref{fig:GDPs}(a). 
For node feature prompt, it introduces prompt embeddings into the node feature space to better align the pre-trained GNN. 
For example, Fang et al.~\cite{gpf24} propose a simple graph feature prompt feature (GPF) method by adding learnable prompts to all input node features.
Li et al.~\cite{li2024instance} develop Instance-Aware Graph Prompt Learning (IA-GPL) that introduces distinct
prompts tailored to input node features.
%
Liu et al.~\cite{graphprompt23} present a novel prompting strategy that designs a task-specific prompt to guide representation learning in readout operation. 
For edge feature prompt, it aims to augment edge-level attribute information with learnable vectors. 
For example, Fu et al.~\cite{EdgePrompt} develop an effective edge prompt method (EdgePrompt) by adding learnable prompt vectors to edge features.
For edge weight prompt, it adapts or re-weights graph topology by modifying edge strengths. 
Guo et al.~\cite{AAGOD} propose to superimpose a parameterized amplifier matrix on the adjacency matrix of the original input graph for out-of-distribution detection. 
Huang et al.~\cite{uniprompt} present a new graph data prompting (UniPrompt) framework by imposing a topological prompt on the original graph structure.
Fu et al.~\cite{graphtop25} present a graph topology-oriented prompting (GraphTOP) method by reformulating topology-oriented prompt as an edge rewiring process.
%
%
\begin{figure*}[!ht]
\centering
\includegraphics[width=0.98\textwidth]{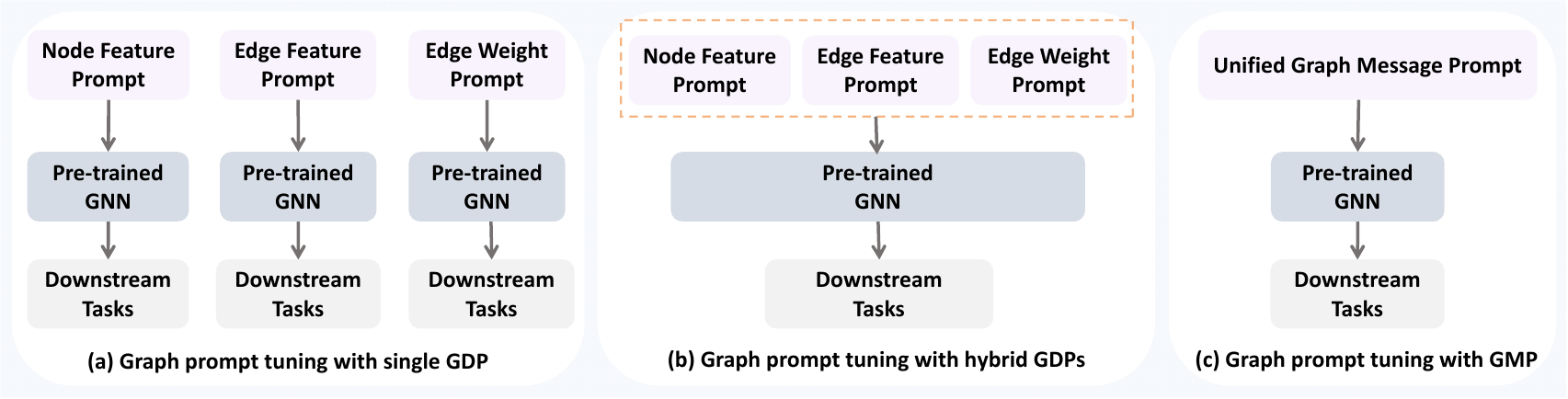}
\caption{Comparison between  Graph Data Prompts (GDPs) and our Graph Message Prompt (GMP).  (a) Using single GDP method to fine-tune pre-trained GNN model;
(b) Using hybrid prompt method that integrates different GDPs to fine-tune pre-trained GNN model; 
(c) Using unified graph message prompt method to fine-tune pre-trained GNN model.}
\label{fig:GDPs}
\end{figure*}
\IEEEpubidadjcol

After revisiting the above various Graph Data Prompts (GDPs), we observe that current approaches are tailored to specific graph components (e.g., node features, edge features or edge weights) and thus operate within limited prompt spaces for graph data. \textbf{Therefore, it is necessary to construct a universal framework capable of performing prompt learning across all graph components simultaneously}.   
Recent works~\cite{vgp,FPrompt} attempt to overcome this limitation by adopting hybrid prompting strategies, which integrate multiple types of GDPs to enhance the expressiveness of graph prompt tuning,
as illustrated in Fig.~\ref{fig:GDPs}(b). 
For example, Li et al.~\cite{FPrompt} integrate node feature prompts with structural constraint prompts to balance tasks performance and fairness objectives. Ai et al.~\cite{vgp} design a hybrid prompt by introducing different prompt generation modules for vision-based graph adaptation scenarios.
However, such hybrid prompt learning inevitably brings increased optimization burden and introduces additional design challenges in determining the optimal combination of different GDPs.
Therefore, this naturally motivates us to raise a question: \textit{Can we develop a unified graph prompting framework that inherently operates on all graph components simultaneously and constructs a rich, expressive prompt space for diverse downstream tasks?}

To address this problem, in this paper, we first revisit graph prompt learning from the perspective of GNN's message passing~\cite{gilmer2017neural,nguyen2026magprompt,zhanguga}. 
Our main observation is that existing Graph Data Prompts (GDPs) operating on graph data (node features, edge features, edge weights, etc.)  can be universally formulated as a kind of Graph Message Prompt (GMP) within GNN's message propagation process.  
GMP not only offers a new perspective to understand a wide range of existing GDPs, but also serves as a unified prompting framework for graph fine-tuning tasks. 
Moreover, based on GMP, we then introduce a novel graph prompt learning approach called Low-Rank GMP (LR-GMP), which leverages low-rank prompt representation to achieve an effective and compact graph prompt
learning model. 
Unlike traditional GDPs that target distinct
graph components separately, LR-GMP simultaneously performs
prompting on all graph components in a unified manner, thereby providing a rich and expressive prompt space 
for downstream graph learning tasks.
Extensive experiments on multiple benchmark datasets demonstrate the effectiveness, superior generalization, and robustness of our proposed LR-GMP framework against state-of-the-art graph prompting baselines. 

Overall, the main contributions of this paper are summarized as follows:
\begin{itemize}
    \item We propose revisiting various Graph Data Prompts (GDPs) from the perspective of the Graph Message Prompt (GMP) scheme by theoretically demonstrating that existing GDPs are special cases of GMP.  
    %
    \item We introduce LR-GMP, a novel graph prompt learning method that uses low-rank prompt representation to enable more effective prompting. It unifies prompting across all graph components, providing a richer and more expressive prompt space for downstream tasks.  
    
   \item We demonstrate the effectiveness, generalization and robustness of the proposed LR-GMP on various benchmark graph fine-tuning tasks. 
\end{itemize}

The remainder of this paper is organized as follows. In Section II, we  present some related work on graph pre-training and graph prompt learning. In Section III, we review the definition of graph prompt tuning and existing GDPs. 
In Section IV, we propose our GMP formulation and further develop LR-GMP scheme. We show that our method is capable of unifying various GDPs in Section V. Finally, we evaluate our method on diverse graph learning tasks in Section VI.

\section{Related Works}
\subsection{Graph Pre-training}
To alleviate the reliance on labeled supervision, substantial studies
have focused on developing self-supervised pre-training strategies for GNNs and enabling label-free graph representation learning. 
For example, 
GCC~\cite{qiu2020gcc} designs the pre-training objective as a subgraph-level instance discrimination task to learn generalizable structural representations across diverse graph instances.
GraphCL~\cite{graphcl} proposes a contrastive pre-training framework that encourages consistency between multiple augmented views derived from the same graph.
GraphMAE~\cite{GraphMAE} proposes a generative pre-training framework based on a masked graph autoencoder.
SGCL~\cite{SGCL} proposes a simplified self-supervised framework that learns graph representations through latent prediction without relying on negative samples.
GMCL~\cite{zhanguga} proposes a universal message  contrastive learning scheme to reformulate existing graph augmentation strategies.
SUGRL~\cite{mo2022simple} presents a simple and effective self-supervised framework to achieve effective and efficient graph contrastive learning.

\subsection{Graph Prompt Tuning}
Existing graph data prompt methods fall into the following categories: node feature prompt, edge feature prompt, edge weight prompt, subgraph prompt, or hybrid prompt. 
For example, GPF~\cite{gpf24} proposes to conduct prompting on node feature space by introducing learnable vectors. 
GPC~\cite{gpt_cluster} introduces a group of learnable prompts into graph-level representations during fine-tuning stage.
RELIEF~\cite{RELIEF} adopts a reinforcement learning framework to perform selective prompting on node features. 
EdgePrompt~\cite{EdgePrompt} attempts to introduce learnable prompts to edges for graph fine-tuning. 
AGOD~\cite{AAGOD} superimposes learnable amplifiers on adjacency matrices to highlight OOD-relevant substructures.
GraphTOP~\cite{graphtop25} proposes a graph topology-oriented prompting framework by modifying topology structure for graph learning tasks.
UniPrompt~\cite{uniprompt} also introduces a graph structure prompt to adjust edge weights for downstream
few-shot tasks.
All in One~\cite{all_in_one23} designs a novel prompting framework that integrates the prompted subgraph into original graph to allow flexible adaptation to downstream tasks.
MAGPrompt~\cite{nguyen2026magprompt} proposes a message-adaptive graph prompting method that performs message passing through a gating mechanism and a learnable additive prompt. 
FPrompt~\cite{FPrompt} introduces hybrid graph prompts to bridge the gap between pre-training and downstream tasks.
VGP~\cite{vgp} employs three different graph data prompts based on semantic prompt generation modules for vision GNN learning.

\section{Preliminaries}
\subsection{Problem Formulation}
\label{subsec:problem_formulation}

Let $\mathcal{G} = (\mathcal{V}, \mathcal{E}, \mathbf{E}, \mathbf{H}, \mathbf{A})$ be an input graph, where $\mathcal{V}$ and $\mathcal{E}$ denote the set of nodes and edges, respectively. $ \mathbf{E} \in \mathbb{R}^{|\mathcal{E}| \times d_\mathcal{E}} $ and $ \mathbf{H} \in \mathbb{R}^{N \times d_\mathcal{V}} $ represent the collection of edge features and node features, respectively, where $|\mathcal{E}|$ and $N$ represent the number of edges and nodes, $d_\mathcal{E}$ and $d_\mathcal{V}$ indicate the dimensions of the features. $ \mathbf{A} =[\mathbf{A}_{vu}]\in \{0,1\}^{N \times N}$ represents the adjacency matrix. 

\textbf{Graph Pre-training.} Given large-scale pre-training graph data $\mathcal{G}_{pre}$,  a graph neural network (GNN) backbone is pre-trained in a self-supervised way to learn transferable and generalizable representations. 
Then, the pre-training objective can generally be defined as
\begin{equation}
\min_{\Theta}\,\, \mathcal{L}_{pre}\big(\text{GNN}_{\Theta}(\mathcal{G}_{pre},\mathcal{D})\big),
\label{eq:pretraining_objective}
\end{equation}
where $\mathcal{D}$ represents the self-supervised signals derived from the distribution of unlabeled graph samples in the input graph space $\mathcal{G}_{pre}$ and $\Theta$ represents the collection of all trainable parameters of the GNN backbone.

\textbf{Graph Prompt Tuning.}
For an input downstream graph $\mathcal{G}$, 
graph prompt tuning aims to adapt the pretrained model $\text{GNN}_{\Theta^*}$ to the downstream task by introducing a small set of learnable prompt parameters $\Omega$. Formally, the prompt tuning objective can be formulated as
\begin{equation}
\min_{\Omega,\pi} \mathcal{L}_{down}\big(g_{\pi}\big(\text{GNN}_{\Theta^*}(\tilde{\mathcal{G}}), {Y} \big)\big)\,,\,\,\, s.t.\,\,\, \tilde{\mathcal{G}} = \mathcal{P}(\mathcal{G}, \Omega)
\label{eq:prompt_tuning_objective}
\end{equation}
where $Y$ represents the set of labeled data for model fine-tuning and $\tilde{\mathcal{G}}$ represents the prompted graph data. $g_{\pi}$ is a trainable task classifier parameterized by $\pi$.  $\mathcal{P}(\cdot)$ denotes the prompt function and $\Omega$ denotes the collection of prompt parameters defined on different graph components.
The key goal of prompt tuning is to design an effective and efficient prompt function $\mathcal{P}(\cdot)$ that enables {compact, adaptive and task-aware} graph prompting on the downstream tasks.

\subsection{Graph Data Prompts (GDPs)}
A mainstream direction in graph prompt tuning is to develop graph-specific prompt function that reformulates either the input graph or its latent representation on the downstream tasks.  Existing Graph Data Prompt methods (GDPs) can generally be divided into five categories, i.e., node feature prompt, edge feature/weight prompt, subgraph prompt, or hybrid prompt. In the following, we briefly review them, respectively. 

\textbf{Node feature prompt}  attempts to inject prompt embeddings into the node feature space to align the pre-trained GNN’s feature distribution with downstream tasks~\cite{graphprompt23,gpf24,RELIEF}.   
One simple and efficient way is to directly introduce a single learnable vector~\cite{graphprompt23,gpf24} to implement the shared feature prompt as follows: 
\begin{equation}\label{EQ:fpt1}
    \mathbf{\tilde{H}} = \mathcal{P}(\mathbf{H},\bm{z})= \mathbf{H} + \bm{1}\bm{z}^{\top}
\end{equation}
where $\bm{z}\in \mathbb{R}^{d_{\mathcal{V}}}$ is the learnable prompt vector and $\mathbf{\tilde{H}}$ denotes the prompted node features. $\bm{1}\in\mathbb{R}^{N}$ is an all-ones vector. 
Furthermore, to generate more expressive and discriminative prompted features, an effective way is to adopt multiple learnable vectors to improve the capacity of feature prompting method~\cite{gpf24,RELIEF} as follows:
\begin{align}\label{EQ:fpt2}
    &\mathbf{\tilde{H}} = \mathcal{P}(\mathbf{H},\mathbf{Z})= \mathbf{H} + \bm{\alpha}^{n}\mathbf{Z} \\
    &\mathrm{where} \,\,\, \bm{\alpha}^{n}=\text{Softmax}\big(\mathbf{H}\, {\mathbf{Z}}^{\top}/\tau\big)\nonumber
\end{align}
where $\mathbf{Z}\in \mathbb{R}^{k\times d_{\mathcal{V}}}$ contains $k$ bias prompt vectors and $\bm{\alpha}^{n}\in \mathbb{R}^{N\times k}$ indicates the assignment matrix that maps bias prompts to each node feature.

\textbf{Edge feature prompt}  extends the idea of feature modulation to the edge domain, which modifies the content of propagated messages by augmenting edge-level attribute information~\cite{EdgePrompt}. Specifically, edge feature prompt can be implemented by adding a learnable vector to edge information as follows:  
\begin{equation}\label{EQ:ept1}
    \mathbf{\tilde{E}} = \mathcal{P}(\mathbf{E},\bm{f})= \mathbf{E}+\bm{1}\bm{f}^{\top}
\end{equation}
where  $\bm{f}\in \mathbb{R}^{d_{\mathcal{E}}}$ denotes the learnable prompt vector and $\mathbf{\tilde{E}}$ denotes the prompted edge features. Besides, a more expressive variant~\cite{EdgePrompt} employs multiple prompts for edge adaptation as
\begin{align}\label{EQ:ept2}
    &\mathbf{\tilde{E}} = \mathcal{P}(\mathbf{E},\mathbf{F})=\mathbf{E}+\bm{\alpha}^{e}\mathbf{F} \\
    &\mathrm{where} \,\,\, \bm{\alpha}^{e}=\text{Softmax}\big(\mathbf{E}\, \mathbf{F}^{\top}/\tau\big)\nonumber
\end{align}
where $\mathbf{F}\in \mathbb{R}^{k\times d_{\mathcal{E}}}$ denotes $k$ bias prompt vectors for all edges and $\bm{\alpha}^{e}\in \mathbb{R}^{|\mathcal{E}|\times k}$ indicates the assignment matrix that maps bias prompts to each edge.

\textbf{Edge weight prompt}  aims to directly modulate the strength of propagated information along graph edges~\cite{graphtop25}. Formally, we can generate a prompted edge weight by introducing the learnable structural bias~\cite{AAGOD,uniprompt} as
\begin{equation}\label{EQ:ept3}
    \mathbf{\tilde{A}} = \mathcal{P}(\mathbf{A},\mathbf{S})= \mathbf{A}+\mathbf{S} 
\end{equation}
where $\mathbf{S}\in \mathbb{R}^{N\times N}$ encodes the edge weight prompts and $\mathbf{\tilde{A}}$ denotes the prompted adjacency matrix. 
In addition, one can also achieve the edge weight prompt by using learnable dynamic scaling~\cite{graphtop25} as follows:
\begin{equation}\label{EQ:ept4}
    \mathbf{\tilde{A}} = \mathcal{P}(\mathbf{A},\mathbf{S})= \mathbf{A}\odot \mathbf{S}
\end{equation}
where $\odot$ indicates the element-wise multiplication.

\textbf{Subgraph prompt} achieves a  structural graph prompt, which enhances the original graph by inserting an additional prompt graph~\cite{all_in_one23,vnt23,psp}. 
Given a prompt graph with $K$ nodes, we can generally define the subgraph prompt as follows:
\begin{equation}\label{eq:sub_p}
  \tilde{\mathcal{G}} = \mathcal{P}(\mathcal{G}, \mathcal{G}^{p})=\mathcal{G}\cup\mathcal{G}^{p},\,\,\,\,\mathcal{G}^{\,p} = (\mathcal{V}^{p}, \mathcal{E}^{p}, \mathbf{E}^{p}, \mathbf{H}^{p}, \mathbf{A}^{p})  
\end{equation}
where $\tilde{\mathcal{G}}$ is the prompted graph and $\mathcal{V}^{p}$ and $\mathcal{E}^{p}$ indicate the sets of prompt nodes and edges. 
$\mathbf{H}^{p}\in \mathbb{R}^{K\times d_{\mathcal{V}}}$ and $\mathbf{E}^{p}\in \mathbb{R}^{|\mathcal{E}^{p}|\times d_{\mathcal{E}}}$ denote the node and edge embeddings, respectively. $\mathbf{A}^{p}\in \mathbb{R}^{K\times K}$ encodes the link structure among these prompt nodes. 

\textbf{Hybrid prompt} attempts to integrate multiple GDPs into a single prompt tuning framework to adapt pre-trained model~\cite{FPrompt,vgp}. For example, FPrompt~\cite{FPrompt} introduces node-level and structural prompts to improve both the fairness and effectiveness of the pre-trained model. Inspired by this design in FPrompt, hybrid prompting can generally be formulated  as jointly applying prompts to node features and edge weights as follows:
\begin{equation}\label{EQ:hyb_p}
     \mathcal{G}(\mathbf{{E}}, \mathbf{\tilde{H}}, \mathbf{\tilde{A}}) = \mathcal{G}\big(\mathbf{{E}}, \mathcal{P}(\mathbf{H},\mathbf{Z}), \mathcal{P}(\mathbf{A},\mathbf{S})\big) 
\end{equation}
Here, $\mathcal{P}(\mathbf{H},\mathbf{Z})$ and $\mathcal{P}(\mathbf{A},\mathbf{S})$ are defined in Eq.(\ref{EQ:fpt2}) and Eq.(\ref{EQ:ept4}), respectively.  
Then, the prompted graph data is fed into the pre-trained backbone to perform downstream tasks.

\section{Methodology}
In this section, we first present a unified Graph Message Prompt (GMP) and further develop a graph message prompt learning schema (LR-GMP) that leverages low-rank prompt representations. Then, we present the process of our LR-GMP tuning  framework in Fig.~\ref{fig:GMPT}. 
\begin{figure*}[!ht]
\centering
\includegraphics[width=1\textwidth]{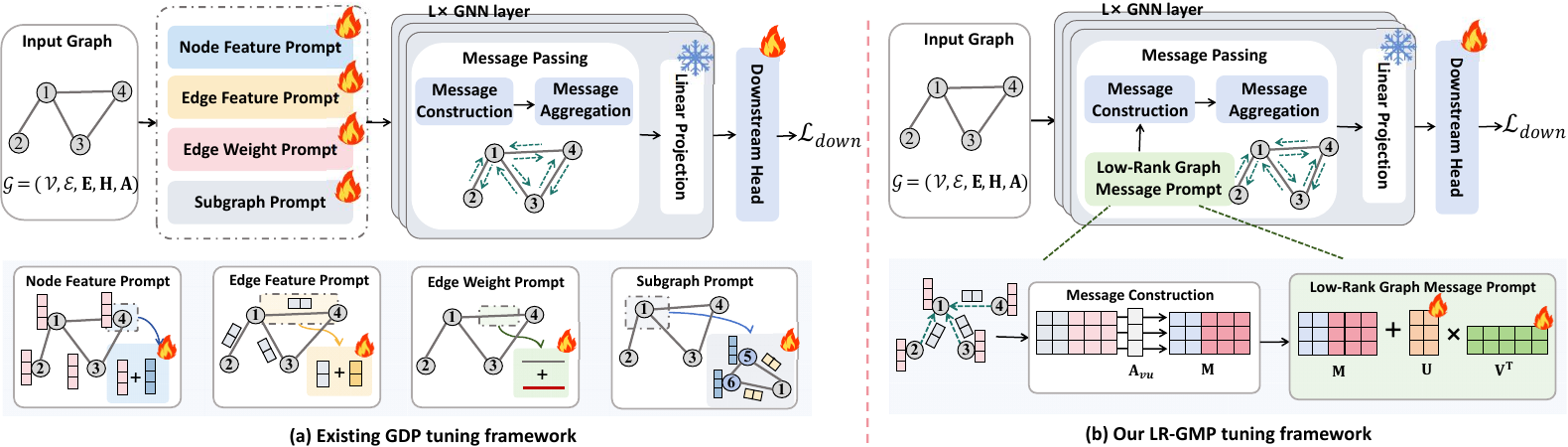}
\caption{Comparison of our proposed LR-GMP tuning framework and existing GDPs tuning. By jointly modulating node feature and structural information in the message space, LR-GMP unifies diverse prompting strategies into a single paradigm.}
\label{fig:GMPT}
\end{figure*}

\subsection{Graph Message Representation}
Given an input graph $\mathcal{G}(\mathcal{V}, \mathcal{E}, \mathbf{E}, \mathbf{H}, \mathbf{A})$, $\mathbf{E} \in \mathbb{R}^{|\mathcal{E}| \times d_\mathcal{E}}$ and $\mathbf{H} \in \mathbb{R}^{N \times d_\mathcal{V}}$ denote edge features and node features, and $\mathbf{A} \in \{0,1\}^{N \times N}$ denotes the adjacency matrix. GNNs generally learn graph representation via a multi-layer message passing mechanism~\cite{velivckovic2022message,zhanguga} that iteratively updates each node representation by aggregating the information from their neighbors. In general, 
the message passing process is formally  defined as
\begin{align}\label{eq:MC}
 \mathbf{H}'_{v} \leftarrow \mathcal{A}gg\Big( \mathbf{H}_v, \,\mathcal{M}(\mathbf{H}_u, \mathbf{E}_{vu}, \mathbf{A}_{vu}) \,\,|\,\, \forall u \in \mathcal{N}_v\Big)   
\end{align}
where $\mathcal{M}(\cdot)$ and $\mathcal{A}gg(\cdot)$ denote the message construction  and aggregation functions, respectively.  
For simplicity, let $\mathbf{M}_{v \leftarrow u} = \mathcal{M}(\mathbf{H}_{u}, \mathbf{E}_{vu}, \mathbf{A}_{vu})$. Then, we can rewrite Eq.(\ref{eq:MC}) as follows:
\begin{align}\label{eq:AGG}
 \mathbf{H}'_{v} \leftarrow \mathcal{A}gg\big( \mathbf{H}_v, \,\mathbf{M}_{v \leftarrow u} \,\,|\,\, \forall u \in \mathcal{N}_v\big)   
\end{align}
%


Without loss of generality, we assume that missing edge features are padded with zeros. This allows all message functions to be uniformly represented in concatenated form. 
Below, we briefly revisit the message function $\mathcal{M}(\cdot)$ in several popular GNN models. 
For example, MPNN~\cite{gilmer2017neural} adopts the simple summation aggregation and thus the message function is formulated as follows:
\begin{equation}\label{MSG1}
\mathbf{M}_{v \leftarrow u}=\mathcal{M}(\mathbf{H}_u, \mathbf{E}_{vu}, \mathbf{A}_{vu})= \mathbf{A}_{vu}[\mathbf{H}_u \Vert\mathbf{E}_{vu}] 
\end{equation}
where $ \,\|\, $ denotes concatenation. 
GCN~\cite{kipf2017semi} employs Laplacian-normalized edge weights to mitigate the scale discrepancy of neighbor information and thus the message function can be defined as
\begin{equation}\label{eq:MSG2}
\mathbf{M}_{v \leftarrow u}=\mathcal{M}(\mathbf{H}_u, \mathbf{E}_{vu}, \mathbf{A}_{vu}) = \mathbf{\hat{A}}_{vu} \, \mathbf{H}_u
\end{equation}
where $\mathbf{\hat{A}}_{vu}$
denotes the normalized edge weights. 
In GAT~\cite{velickovic2018graph}, it introduces a self-attention module to obtain learnable coefficients to re-weight the contribution of neighboring nodes. Then, the message function is defined as
\begin{equation}\label{eq:MSG3}
\mathbf{M}_{v \leftarrow u}=\mathcal{M}(\mathbf{H}_u, \mathbf{E}_{vu}, \mathbf{A}_{vu}) = \bm{\alpha}_{vu} \, \mathbf{H}_u
\end{equation}
where $\bm{\alpha}_{vu}$ denotes the  learned attention coefficient. 

\subsection{Graph Message Prompt Scheme}

Building upon the message representation in Eq.(\ref{eq:MC}), we can propose a more general graph data prompting strategy, termed Graph Message Prompt (GMP), which performs prompting on graph messages rather than on the original graph data. 
Specifically,  we define the message matrix as
$$
\mathbf{M} = \big\{\mathbf{M}_{v \leftarrow u}|\forall u \in \mathcal{N}_v,v=1\dots N\big\}\in \mathbb{R}^{|\mathcal{E}|\times (d_\mathcal{V}+d_\mathcal{E})}
$$
%
Then, GMP can generally be formulated as
\begin{equation}
\mathbf{\tilde{M}} = \mathcal{P}(\mathbf{M},\mathbf{P})
\end{equation}
where $\mathbf{P}$ denotes the learnable message prompt and 
$\mathcal{P}(\cdot)$ represents the prompt function, such as summation or element-wise multiplication. For instance, we can implement additive message prompting as follows:
\begin{equation}
\mathbf{\tilde{M}} = \mathcal{P}(\mathbf{M},\mathbf{P})=\mathbf{M}+\mathbf{P}
\label{eq:gmp}
\end{equation}

One key property of the above GMP is its ability to unify various Graph Data Prompts (GDPs), including node feature prompt (Eqs.(\ref{EQ:fpt1},\ref{EQ:fpt2})), edge feature prompt (Eqs.(\ref{EQ:ept1},\ref{EQ:ept2})), edge weight prompt (Eqs.(\ref{EQ:ept3},\ref{EQ:ept4})),  subgraph prompt (Eq.(\ref{eq:sub_p})) and hybrid prompt (Eq.(\ref{EQ:hyb_p})). In other words, existing GDPs can be viewed as specific instances of GMP by appropriately defining the prompt matrix $\mathbf{P}$. 
We will show this property in Section V in detail. Below, we present a novel specific graph message prompt method. 

\subsection{Low-Rank GMP}

As we know, the prompt component is typically encouraged to be low-rank~\cite{Lora,LoR-VP}. 
In the extreme case, the prompt can be formulated as $\mathbf{P}=\bm{1}^{\top}\bm{p}$  where $\bm{p}$ is a single vector~\cite{gpf24,EdgePrompt,nguyen2026magprompt}. 
In this case $rank(\mathbf{P})=1$.  
Based on the above observation, 
we propose Low-Rank GMP (LR-GMP) by employing low-rank prompt factorization on GMP.  
 Specifically, let $\mathbf{U}\in \mathbb{R}^{|\mathcal{E}|\times r}$ and $\mathbf{V}\in \mathbb{R}^{(d_\mathcal{V}+d_\mathcal{E}) \times r}$ denote low-rank prompt factors. Then, our LR-GMP function is defined as follows:
    \begin{equation}
    \mathbf{\tilde{M}} = \mathcal{P}_{r}(\mathbf{M},\,\mathbf{U},\,\mathbf{V}) = \mathbf{M} + \mathbf{U}\mathbf{V}^{\top}
    \label{eq:LR-GMP}
    \end{equation}
    where $r$ is the rank dimension to control the prompt capacity and $\mathcal{P}_{r}(\cdot)$ represents the low-rank prompt projection operator. 
In practice, we directly instantiate the low-rank factors $\mathbf{U}\in \mathbb{R}^{|\mathcal{E}|\times r}$ and $\mathbf{V}\in \mathbb{R}^{(d_\mathcal{V}+d_\mathcal{E}) \times r}$ in Eq.(\ref{eq:LR-GMP}) as two learnable parameter matrices. 
For inductive graph learning scenario, 
we can also learn the instance-specific low-rank factor $\mathbf{U}$ in Eq.(\ref{eq:LR-GMP}) based on the input graph message as follows:
    \begin{equation}
        \mathbf{U} = \phi(\mathbf{M}; \mathbf{W})
        \label{eq:dynamic_u}
    \end{equation}
where $\phi(\cdot)$ is a lightweight projection parameterized by $\mathbf{W}\in \mathbb{R}^{(d_\mathcal{V}+d_\mathcal{E})\times r}$. 
This design can be interpreted as a conditional low-rank adaptation mechanism, which extends the above low-rank prompt (Eq.(\ref{eq:LR-GMP})) into a more flexible adapter-based paradigm.

\section{Unifying GDPs from GMP}

In this section, we demonstrate that the effects of various GDPs can be represented as message-level offsets. This indicates that the proposed GMP is capable of unifying diverse GDPs, thereby enabling a rich and expressive prompt space. 
To be specific, 
we show that existing GDPs (including node feature prompt, edge feature prompt, edge weight prompt, subgraph prompt and hybrid prompt) can be equivalently represented via GMP. 
In the following proof, the message is defined as 
$$
\mathbf{M}_{v \leftarrow u} = \mathbf{A}_{vu}[\mathbf{H}_u \Vert \mathbf{E}_{vu}]
$$ 
We have the following Proposition 1-5. 

\begin{proposition}
For node feature prompt, there exists a special GMP that achieves the equivalent effect. 
\end{proposition}



\begin{proof}
 As shown in Eq.(\ref{EQ:fpt1}), 
 conducting node feature prompting with a single prompt vector $\bm{z} \in \mathbb{R}^{ d_\mathcal{V}} $ is expressed as $\mathbf{\tilde{H}} = \mathbf{H} + \bm{z}$. 
It is the same as conducting prompt learning in graph message space as
\begin{equation}\label{eq:P1-1}
    \begin{aligned}
        \mathbf{\tilde{M}}_{v\leftarrow u} 
        &= \mathbf{A}_{vu} [\mathbf{H}_u \Vert \mathbf{E}_{vu}] + \mathbf{A}_{vu} [\bm{z} \Vert \mathbf{0}]\\
       & =\mathbf{{M}}_{v\leftarrow u}+\mathbf{A}_{vu}[\bm{z} \Vert \mathbf{0}]
    \end{aligned}
\end{equation} 
 %
    That is, it is 
equivalent to conducting
GMP (Eq.(\ref{eq:gmp})) by defining  the message prompt as 
\begin{equation}
\mathbf{P}_{v\leftarrow u} = \mathbf{A}_{vu} [\bm{z}\Vert \mathbf{0}]
\end{equation} 
    
 As shown in Eq.(\ref{EQ:fpt2}), given multiple basis prompt $\mathbf{Z}\in \mathbb{R}^{k\times d_{\mathcal{V}}}$, the prompted node feature is expressed as $\mathbf{\tilde{H}}_u = \mathbf{H}_u + \sum_{j}^{k}\bm{\alpha}_{u,j}^{n}\mathbf{Z}_{j}$. Similarly, this is  equivalent to conducting 
GMP by defining  the message prompt 
as 
    \begin{equation}
    \mathbf{P}_{v\leftarrow u} = \mathbf{A}_{vu} \big[\sum_{j=1}^{k}\bm{\alpha}_{uj}^{n}\mathbf{Z}_{j} \ \Vert \mathbf{0}\big]
    \end{equation} 
    where $\bm{\alpha}_{u,j}^{n}$ represents the learnable coefficient for the $j$-th basis node prompt associated with node $u$.  
\end{proof}

\begin{proposition}
For edge feature prompt, there exists a special GMP that achieves the equivalent effect. 
\end{proposition} 
\begin{proof}
As shown in Eq.(\ref{EQ:ept1}), edge feature prompt injects a learnable vector $\bm{f} \in \mathbb{R}^{ d_\mathcal{E}} $ into edge feature which is formulated as $\mathbf{\tilde{E}}_{v\leftarrow u} = \mathbf{E}_{vu} + \mathbf{f}$.  
 It is equivalent to conducting prompt learning in message space as  \begin{equation}\label{eq:P2-1}
    \begin{aligned}   \mathbf{\tilde{M}}_{v\leftarrow u} &= \mathbf{A}_{vu} [\mathbf{H}_u \Vert \mathbf{E}_{vu}] + \mathbf{A}_{vu}[\mathbf{0} \Vert \bm{f} ]\\
        & =\mathbf{{M}}_{v\leftarrow u}+\mathbf{A}_{vu} [\mathbf{0} \Vert \bm{f} ]
    \end{aligned}
\end{equation} 
    This is equivalent to conducting GMP (Eq.(\ref{eq:gmp})) by 
    defining the message prompt  as 
    \begin{equation}
    \mathbf{P}_{v\leftarrow u} = \mathbf{A}_{vu}  [\mathbf{0} \Vert \bm{f}]
    \end{equation} 
    
    As shown in Eq.(\ref{EQ:ept2}), given multiple prompt $\mathbf{F}\in \mathbb{R}^{k\times d_{\mathcal{E}}}$, the prompted edge feature is defined as $\mathbf{\tilde{E}}_{vu} = \mathbf{E}_{vu} +  \sum_{j}^{k}\bm{\alpha}_{vu,j}^{e}\mathbf{F}_{j}$. From Eq.(\ref{eq:P2-1}), this is mathematically equivalent to conducting GMP by defining the message prompt as  
    \begin{equation}
    \mathbf{P}_{v \leftarrow u} = \mathbf{A}_{vu}\big[\mathbf{0} \Vert \sum_{j=1}^{k}\bm{\alpha}_{vu,j}^{e}\mathbf{F}_{j}\big]
    \end{equation} 
    where $\bm{\alpha}_{vu,j}^{e}$ represents the learnable coefficient for the $j$-th edge prompt associated with edge from node
    $ v$ to $u$.
\end{proof}

\begin{proposition}
For edge weight prompt, there exists a special GMP that achieves the equivalent effect. 
\end{proposition} 
\begin{proof}
    As shown in Eqs.(\ref{EQ:ept3},\ref{EQ:ept4}), edge weight prompt can adopt an additive bias or multiplicative scaling approach to modulate structural connectivity strength during message construction.
    For the additive strategy in Eq.(\ref{EQ:ept3}), given the edge weight prompt $ \mathbf{S} \in \mathbb{R}^{N\times N} $, it is the same as conducting prompt learning in graph message space as follows:
    \begin{equation}\label{eq:P3-1}
         \begin{aligned}
       \mathbf{\tilde{M}}_{v\leftarrow u}& = (\mathbf{A}_{vu}+\mathbf{S}_{vu}) [\mathbf{H}_u \Vert \mathbf{E}_{vu}]  \\
       &=\mathbf{{M}}_{v\leftarrow u}+\mathbf{S}_{vu}[\mathbf{H}_u \Vert \mathbf{E}_{vu}]   
    \end{aligned}
    \end{equation}
   This is equivalent to conducting GMP (Eq.(\ref{eq:gmp})) by defining the message prompt as 
   \begin{equation}
   \mathbf{P}_{v \leftarrow u} = \mathbf{S}_{vu}[\mathbf{H}_u \Vert \mathbf{E}_{vu}]
   \end{equation}
   Similarly, the multiplicative scaling prompt in Eq.(\ref{EQ:ept4}) is equivalent to performing prompt in graph message space as 
    \begin{equation}\label{eq:P3-2}
        \begin{aligned}
          \mathbf{\tilde{M}}_{v\leftarrow u}& = \mathbf{A}_{vu}\mathbf{S}_{vu}[\mathbf{H}_u \Vert \mathbf{E}_{vu}]  \\
       &=  \mathbf{{M}}_{v\leftarrow u} + (\mathbf{\tilde{A}}_{vu} - \mathbf{A}_{vu})[\mathbf{H}_{u}\Vert\mathbf{E}_{vu}]  
        \end{aligned}
    \end{equation}
    The above Eq.(\ref{eq:P3-2}) can be equivalent to conducting GMP  by defining the message prompt as
    \begin{equation}
    \mathbf{P}_{v\leftarrow u}=(\mathbf{\tilde{A}}_{vu} - \mathbf{A}_{vu}) [\mathbf{H}_{u}\Vert\mathbf{E}_{vu}]
    \end{equation}
    where $\mathbf{\tilde{A}}_{vu}$ is the prompted edge weight by the multiplicative prompting mechanism of $\mathbf{\tilde{A}}= \mathbf{A}\odot\mathbf{S}$.
   
\end{proof}

\begin{proposition}
\label{prop:subgraph}
For subgraph prompt $\mathcal{G}^p=(\mathcal{V}^p, \mathcal{E}^p)$ applied to the original graph $\mathcal{G}$, there exists a special GMP that achieves an approximate effect based on the message aggregation.
\end{proposition}

\begin{proof}
As shown in Eq.(\ref{eq:sub_p}), subgraph prompt augments the original graph $\mathcal{G}$ by inserting a prompt subgraph $\mathcal{G}^p$, resulting in $\tilde{\mathcal{G}}=\mathcal{G}\cup\mathcal{G}^p$. To be specific, the messages received by a node $v \in \mathcal{V}$ consist of information aggregated from its neighborhood set in the original graph and from its neighborhood set in the prompted subgraph.
Then, the message aggregation for node $v$ (Eq.(\ref{eq:AGG})) can be rewritten as follows:
 \begin{equation}
   \begin{aligned}\label{eq:sgmp}
      \mathbf{H}'_{v}&= \sum\limits_{u\in\mathcal{N}_v} 
    \mathbf{M}_{v\leftarrow u} +\sum\limits_{k\in\mathcal{N}^p_v} \mathbf{M}_{v\leftarrow k} \\
    &=\sum\limits_{u\in\mathcal{N}_v} \mathbf{A}_{vu}[\mathbf{H}_u\Vert \mathbf{E}_{vu}] +\sum\limits_{k\in \mathcal{N}^p_v}\mathbf{A}^p_{vk}[\mathbf{{H}}^p_k\Vert \mathbf{E}^p_{vk}]\\
    &=\sum\limits_{u\in\mathcal{N}_v} 
    (\mathbf{M}_{v\leftarrow u} +\mathbf{P}_{v\leftarrow u}) \\
   \end{aligned}
    \end{equation}
where $\mathcal{N}_v$ and $\mathcal{N}^p_v$ denote the neighborhood set of node $v$ in the original graph $\mathcal{G}$ and the prompted subgraph $\mathcal{G}^p$, respectively. Therefore, the subgraph prompt is equivalent to conducting GMP (Eq.(\ref{eq:gmp})) by defining the message prompt as
    \begin{equation} 
    \mathbf{P}_{v\leftarrow u}
    =\frac{1}{|\mathcal{N}_v|}\sum\limits_{k\in \mathcal{V}^p} \mathbf{A}^p_{vk}[\mathbf{H}^p_k\Vert \mathbf{E}^p_{vk}]
    \end{equation}
%
\end{proof}

\begin{proposition}
For hybrid prompt, there exists a special GMP to achieve the equivalent effect. 
\end{proposition} 
\begin{proof}
 Hybrid prompt aims to integrate different GDPs to enhance the adaptability of pre-trained models. Here, we consider a hybrid prompt that combines the node feature prompt defined in Eq.(\ref{EQ:fpt2}) and the edge weight prompt defined in Eq.(\ref{EQ:ept4}). Given prompt $\mathbf{Z}$ and $\mathbf{S}$, 
we can decompose the message prompt $\mathbf{P}_{v\leftarrow u}$ into node and edge components, 
\begin{equation}\label{eq:hgmp}
\begin{aligned}
\mathbf{P}_{v\leftarrow u} &= \mathbf{\tilde{M}}_{v\leftarrow u} - \mathbf{M}_{v\leftarrow u} \\
&= \mathbf{\tilde{A}}_{vu}[(\mathbf{H}_u + \mathbf{Z}_u )\Vert \mathbf{E}_{vu}] - \mathbf{A}_{vu}[\mathbf{H}_u \Vert \mathbf{E}_{vu}]\\
&=\mathbf{A}_{vu} \big( (\mathbf{S}_{vu} - 1) [\mathbf{H}_u || \mathbf{E}_{vu}] + \mathbf{S}_{vu} [\mathbf{Z}_u || \mathbf{0}] \big)
\end{aligned}
\end{equation}
where $\mathbf{\tilde{A}}_{vu} = \mathbf{A}_{vu}\mathbf{S}_{vu}$ denotes the prompted edge weight as shown in Eq.(\ref{EQ:ept4}). Therefore, the hybrid prompt is equivalent to conducting GMP (Eq.(\ref{eq:gmp})) by defining the message prompt as shown in Eq.(\ref{eq:hgmp}). 
\end{proof}

\section{Experiments}

\subsection{Dataset Setup}
\textbf{Pre-training.}
For node classification tasks, we take two different large-scale datasets (ogbn-products~\cite{OGB} and Flickr~\cite{graphsaint} datasets) to train the GNN backbone based on different pre-training strategies.
For graph classification tasks, we follow the setting of GPF~\cite{gpf24} and construct a large-scale pre-training dataset comprising $2$ million unlabeled molecular graphs sampled from the ZINC15 database~\cite{sterling2015zinc}.

\textbf{Fine-tuning.}
For node classification tasks, we evaluate our method on six widely used benchmark datasets, including three Citation datasets~\cite{sen2008collective} (Cora, CiteSeer and PubMed), two Amazon co-purchase datasets~\cite{shchur2018pitfalls} (Photo and Computers) and the Coauthor Physics dataset~\cite{shchur2018pitfalls}. 
Then, we perform $1$-shot node classification tasks of one node per class as training labeled data.
For graph classification tasks, we evaluate our methods on seven chemistry molecule datasets provided by previous work~\cite{hustrategies}. Then, we follow GPF~\cite{gpf24} and adopt the challenging scaffold split to generate data splits.
The brief introduction of all datasets  is summarized in Table~\ref{tab:datasets}. 
\begin{table}[!ht]
\centering
\fontsize{7.5pt}{10pt}\selectfont
\caption{Statistics of the benchmark datasets used in our experiments.}
\label{tab:datasets}
\begin{tabular}{c|cccccc}
\hline
{Task Type} & {Dataset} & {Graphs} & {Nodes} & {Edges}  & {Tasks} \\ 
\hline
\multirow{6}{*}{\begin{tabular}[c]{@{}c@{}}Node \\ Classification\end{tabular}} 
 & Cora             & 1           & 2,708        & 5,429   & 7  \\
 & CiteSeer         & 1           & 3,327        & 4,732   & 6 \\
 & PubMed           & 1           & 19,717       & 44,338  & 3 \\
 & Photo     & 1           & 7,650        & 119,081        & 8 \\
 & Computers & 1           & 13,752       & 245,861        & 10 \\
 & Physics & 1           & 34,493       & 247,962          & 5  \\ 
\hline
{Task Type} & {Dataset} & {Graphs} & {Avg.Nodes} & {Avg.Edges}& {Tasks} \\ 
\hline
\multirow{7}{*}{\begin{tabular}[c]{@{}c@{}}Graph \\ Classification\end{tabular}} 
 & BBBP             & 2,039       & 24.1        & 25.9    & 1 \\
 & Tox21            & 7,831       & 18.6         & 19.3   & 12 \\
 & ToxCast          & 8,575       & 18.8         & 19.3   & 617 \\
 & SIDER            & 1,427       & 33.6         & 35.0   & 27   \\
 & ClinTox          & 1,478       & 26.2         & 27.9   & 2  \\
 & MUV              & 93,087      & 24.2         & 25.9   & 17 \\
 & BACE             & 1,513       & 34.1         & 36.9   & 1  \\ 
\hline
\end{tabular}
\end{table}
\begin{table*}[!ht]
	\fontsize{9.3pt}{11pt}\selectfont
    \renewcommand{\arraystretch}{1.2}
	\centering
	\caption{Comparison results of node classification tasks on six benchmarks. The best result is marked in bold.}
	\begin{tabular}{c | l | llllll}
		\toprule	
		\multicolumn{2}{c|}{Method} & Cora & CiteSeer & PubMed & Photo & Computers & Physics\\
		\midrule
		\multirow{11}{*}{\rotatebox{90}{SGCL}} 
        & FT  &  35.80{\scriptsize$\pm$7.94} &  23.65{\scriptsize$\pm$4.19}  & 49.84{\scriptsize$\pm$5.98}  & 54.44{\scriptsize$\pm$8.48}  &  39.58{\scriptsize$\pm$11.09}  &  75.74{\scriptsize$\pm$11.65} \\
		& All in One  &  40.54{\scriptsize$\pm$6.03} &  27.10{\scriptsize$\pm$6.39}  & 51.63{\scriptsize$\pm$9.67}  & 48.63{\scriptsize$\pm$9.25}  &  36.91{\scriptsize$\pm$10.65}  &  71.05{\scriptsize$\pm$8.18}  \\
		& GPF  &  43.68{\scriptsize$\pm$5.74} &  23.08{\scriptsize$\pm$3.35}  & 55.38{\scriptsize$\pm$11.46}  & 52.61{\scriptsize$\pm$12.35} &  39.35{\scriptsize$\pm$9.44}  &  80.36{\scriptsize$\pm$4.46}  \\
		& GPF-plus  &  42.90{\scriptsize$\pm$8.74} &  21.81{\scriptsize$\pm$3.77}  & 55.88{\scriptsize$\pm$11.94}  & 52.07{\scriptsize$\pm$11.14}  &  41.91{\scriptsize$\pm$12.37}  &  80.58{\scriptsize$\pm$3.92}  \\
		& EdgePrompt  & 47.04{\scriptsize$\pm$6.97}  & 28.14{\scriptsize$\pm$6.60} & 56.08{\scriptsize$\pm$6.54}  &  51.50{\scriptsize$\pm$12.72}  &  46.99{\scriptsize$\pm$8.50}  &  77.86{\scriptsize$\pm$6.54} \\
		& EdgePrompt+  & 38.90{\scriptsize$\pm$9.52}  & 25.72{\scriptsize$\pm$6.23} & 52.27{\scriptsize$\pm$8.31}  &  51.64{\scriptsize$\pm$11.23}  &  46.48{\scriptsize$\pm$8.09} &  79.81{\scriptsize$\pm$4.54} \\
		& GraphTOP  & 40.56{\scriptsize$\pm$11.23}  & 33.24{\scriptsize$\pm$3.12} &  52.21{\scriptsize$\pm$9.21}   &  59.01{\scriptsize$\pm$8.74}  &  53.14{\scriptsize$\pm$8.15}  &  80.75{\scriptsize$\pm$4.79} \\
		& UniPrompt  & 45.20{\scriptsize$\pm$10.15}  & 34.76{\scriptsize$\pm$5.91} &  55.17{\scriptsize$\pm$6.25}   &  58.21{\scriptsize$\pm$8.79}  &  56.39{\scriptsize$\pm$12.08} &  81.14{\scriptsize$\pm$12.87} \\
		& LR-GMP  & \textbf{47.96{\scriptsize$\pm$7.42}}  & \textbf{38.76{\scriptsize$\pm$6.40}} & \textbf{62.44{\scriptsize$\pm$2.50}}  & \textbf{60.63{\scriptsize$\pm$9.82}}  & \textbf{59.04{\scriptsize$\pm$9.51}} &  \textbf{84.90{\scriptsize$\pm$4.16}} \\
 	\midrule
        \multirow{11}{*}{\rotatebox{90}{SUGRL}} 
        & FT  &  32.85{\scriptsize$\pm$10.61} &  23.49{\scriptsize$\pm$4.53}  & 45.04{\scriptsize$\pm$4.85}  & 49.47{\scriptsize$\pm$6.77}  &  50.18{\scriptsize$\pm$9.29} &  76.57{\scriptsize$\pm$5.91} \\
		& All in One  &  34.58{\scriptsize$\pm$9.26} &  18.99{\scriptsize$\pm$4.70}  & 49.90{\scriptsize$\pm$7.96}  & 52.35{\scriptsize$\pm$14.80}  &  34.55{\scriptsize$\pm$12.56} & 74.56{\scriptsize$\pm$11.54} \\
		& GPF  &  45.13{\scriptsize$\pm$6.70} & 23.38{\scriptsize$\pm$4.23}  & 55.51{\scriptsize$\pm$11.11}  & 54.49{\scriptsize$\pm$14.73}  &  38.44{\scriptsize$\pm$10.48} &  80.54{\scriptsize$\pm$4.72} \\
		& GPF-plus  &  44.94{\scriptsize$\pm$6.38} &  23.14{\scriptsize$\pm$6.16}  & 55.68{\scriptsize$\pm$11.91}  & 53.21{\scriptsize$\pm$14.54}  &  39.88{\scriptsize$\pm$12.39}  &  80.26{\scriptsize$\pm$4.77} \\
		& EdgePrompt  & 45.58{\scriptsize$\pm$11.21}  & 29.16{\scriptsize$\pm$9.37} & 51.29{\scriptsize$\pm$10.81}  &  53.36{\scriptsize$\pm$11.38}  &  44.93{\scriptsize$\pm$8.41} &  78.92{\scriptsize$\pm$8.12} \\
		& EdgePrompt+  & 41.64{\scriptsize$\pm$10.64}  & 31.40{\scriptsize$\pm$8.55} & 51.40{\scriptsize$\pm$7.22}  &  53.86{\scriptsize$\pm$11.04}  &  46.32{\scriptsize$\pm$8.11}  &  80.01{\scriptsize$\pm$4.56} \\
		& GraphTOP  & 43.46{\scriptsize$\pm$6.35}  & 32.80{\scriptsize$\pm$1.66} &  55.20{\scriptsize$\pm$6.32}   &  58.21{\scriptsize$\pm$9.15}  &  52.26{\scriptsize$\pm$8.46}  &  80.95{\scriptsize$\pm$6.21}\\
		& UniPrompt  & 40.73{\scriptsize$\pm$9.28}  & 36.51{\scriptsize$\pm$4.18} &  54.44{\scriptsize$\pm$10.90}   &  58.19{\scriptsize$\pm$8.18}  &  53.97{\scriptsize$\pm$9.08}  &  82.81{\scriptsize$\pm$7.46} \\
		& LR-GMP  & \textbf{48.16{\scriptsize$\pm$6.69}}  & \textbf{39.31{\scriptsize$\pm$6.26}} & \textbf{61.82{\scriptsize$\pm$2.70}}  & \textbf{59.52{\scriptsize$\pm$5.12}}  & \textbf{57.65{\scriptsize$\pm$10.21}}  &  \textbf{84.85{\scriptsize$\pm$4.30}}  \\
  
	    \midrule
		\multirow{11}{*}{\rotatebox{90}{Edgepred-Gprompt}} 
        & FT  &  36.14{\scriptsize$\pm$5.70}  &  24.80{\scriptsize$\pm$5.36}  & 49.33{\scriptsize$\pm$4.82}  & 55.37{\scriptsize$\pm$9.77}  &  49.25{\scriptsize$\pm$15.44} &  67.56{\scriptsize$\pm$18.97} \\
		& All in One  &  43.94{\scriptsize$\pm$4.23}  & 26.50{\scriptsize$\pm$6.12}  & 51.46{\scriptsize$\pm$9.63}  & 45.80{\scriptsize$\pm$10.52}  &  36.93{\scriptsize$\pm$12.68} &  65.73{\scriptsize$\pm$9.53} \\ 
		& GPF  &  44.26{\scriptsize$\pm$8.17}  &  24.88{\scriptsize$\pm$5.09}  & 55.51{\scriptsize$\pm$11.12}  & 50.35{\scriptsize$\pm$9.75}  &  36.11{\scriptsize$\pm$10.91} &  80.70{\scriptsize$\pm$4.18} \\
		& GPF-plus  &  43.95{\scriptsize$\pm$6.92} &   23.13{\scriptsize$\pm$5.20}  & 55.13{\scriptsize$\pm$11.70}  & 53.37{\scriptsize$\pm$15.75}  &  41.33{\scriptsize$\pm$11.75} &  80.69{\scriptsize$\pm$5.43} \\
		& EdgePrompt  & 48.11{\scriptsize$\pm$7.90}  & 26.38{\scriptsize$\pm$5.32} &  52.22{\scriptsize$\pm$7.43}   &  50.72{\scriptsize$\pm$12.24}  &  50.28{\scriptsize$\pm$7.85} &  80.80{\scriptsize$\pm$4.91}  \\
		& EdgePrompt+  & 44.01{\scriptsize$\pm$10.75}  & 26.40{\scriptsize$\pm$6.51} &  51.72{\scriptsize$\pm$8.53}   &  52.94{\scriptsize$\pm$12.36}  &  47.91{\scriptsize$\pm$8.41} &  80.70{\scriptsize$\pm$5.00} \\
		& GraphTOP  & 42.60{\scriptsize$\pm$8.13}  & 33.86{\scriptsize$\pm$1.66} &  53.69{\scriptsize$\pm$6.12}   &  59.11{\scriptsize$\pm$8.47}  &  52.30{\scriptsize$\pm$8.45}  &  80.80{\scriptsize$\pm$4.35} \\
		& UniPrompt  & 44.17{\scriptsize$\pm$10.22}  & 36.02{\scriptsize$\pm$7.63} &  55.57{\scriptsize$\pm$7.13}   &  59.17{\scriptsize$\pm$6.91}  &  57.28{\scriptsize$\pm$10.56} &  81.03{\scriptsize$\pm$12.41}  \\
		& LR-GMP  & \textbf{50.63{\scriptsize$\pm$8.03}}  & \textbf{40.90{\scriptsize$\pm$7.74}} & \textbf{61.68{\scriptsize$\pm$4.00}}  & \textbf{60.74{\scriptsize$\pm$8.04}} &   \textbf{60.99{\scriptsize$\pm$10.11}} &  \textbf{84.93{\scriptsize$\pm$5.07}} \\
  
	    \midrule
        \multirow{11}{*}{\rotatebox{90}{Edgepred-GPPT}} 
        & FT  &  31.10{\scriptsize$\pm$9.00} &  24.54{\scriptsize$\pm$3.89}  & 47.56{\scriptsize$\pm$5.45}  & 51.59{\scriptsize$\pm$11.86}  &  50.98{\scriptsize$\pm$10.24} &  71.25{\scriptsize$\pm$14.88} \\
        & All in One  &  41.05{\scriptsize$\pm$8.01} &  27.58{\scriptsize$\pm$6.33}  & 50.53{\scriptsize$\pm$9.46}  & 47.03{\scriptsize$\pm$11.06}  &  35.90{\scriptsize$\pm$10.78} &  66.58{\scriptsize$\pm$10.07} \\
        & GPF  &  45.16{\scriptsize$\pm$8.92} &  24.16{\scriptsize$\pm$4.03}  & 55.33{\scriptsize$\pm$11.36}  & 55.10{\scriptsize$\pm$10.91}  &  43.14{\scriptsize$\pm$13.59} &  80.26{\scriptsize$\pm$4.89} \\
        & GPF-plus  &  42.09{\scriptsize$\pm$6.49} &  23.26{\scriptsize$\pm$6.87}  & 55.31{\scriptsize$\pm$11.96}  & 54.52{\scriptsize$\pm$14.32}  &  38.18{\scriptsize$\pm$14.67} &  80.15{\scriptsize$\pm$5.11} \\
        & EdgePrompt  & 47.21{\scriptsize$\pm$7.80}  & 24.65{\scriptsize$\pm$6.24} & 52.22{\scriptsize$\pm$7.43}  &  54.33{\scriptsize$\pm$13.43}  &  50.30{\scriptsize$\pm$9.33} &  80.78{\scriptsize$\pm$4.85}  \\
        & EdgePrompt+  & 43.04{\scriptsize$\pm$3.36}  & 24.68{\scriptsize$\pm$7.18} & 51.10{\scriptsize$\pm$9.06}  &  54.97{\scriptsize$\pm$12.25}  &  47.91{\scriptsize$\pm$10.22} &  80.36{\scriptsize$\pm$4.32}  \\
        & GraphTOP  & 42.64{\scriptsize$\pm$7.57}  & 30.94{\scriptsize$\pm$3.30} &  57.35{\scriptsize$\pm$7.44}   &  58.92{\scriptsize$\pm$7.56}  &  55.65{\scriptsize$\pm$9.12}  &  80.79{\scriptsize$\pm$4.39} \\
        & UniPrompt  & 41.94{\scriptsize$\pm$9.70}  & 36.07{\scriptsize$\pm$6.72} &  55.68{\scriptsize$\pm$6.91}   &  58.46{\scriptsize$\pm$9.50}  &  58.24{\scriptsize$\pm$10.21} &  82.02{\scriptsize$\pm$10.51}  \\
        & LR-GMP  & \textbf{50.34{\scriptsize$\pm$8.96}}  & \textbf{43.57{\scriptsize$\pm$7.73}} & \textbf{62.64{\scriptsize$\pm$1.22}}  & \textbf{60.27{\scriptsize$\pm$8.16}}  & \textbf{60.35{\scriptsize$\pm$8.63}}   &  \textbf{85.08{\scriptsize$\pm$6.29}} \\
        \bottomrule
	\end{tabular}
	\label{ncla}
\end{table*}

\subsection{Parameter Setting}
\textbf{Pre-training.} For node classification tasks, we adopt a three-layer GCN architecture~\cite{kipf2017semi} with a hidden dimension of $128$ as the backbone model. Then, we adopt four different self-supervised strategies to effectively pre-train the GCN model, specifically including SGCL~\cite{SGCL} and SUGRL~\cite{mo2022simple} on ogbn-products dataset, as well as EdgePred-GraphPrompt~\cite{graphprompt23} and EdgePred-GPPT~\cite{gppt22} and on Flickr dataset. In pre-training stage, all hyper-parameters are set according to the default configurations provided in these self-supervised methods.
For graph classification tasks, we adopt a five-layer GIN~\cite{GIN} with a hidden dimension of $300$ as the backbone model. Then, we adopt the EdgePred method~\cite{kipf2017semi} to train GIN backbone. In our experiments, we directly utilize the publicly available pre-trained model provided by GPF~\cite{gpf24}.

\textbf{Fine-tuning.} For node classification tasks, we first reprogram the downstream data by introducing a linear transformation function, which facilitates the reuse of the pre-trained model~\cite{hou2024graphalign}. The output dimension of this transform module is determined based on the pre-trained dataset. In addition, the hidden dimension is set to $128$ to align with the pre-trained model. 
The learning rate and the dropout rate are set respectively to $0.001$ and $0$. 
Finally, the hyper-parameter $r$ is selected from the range of $\{2,5,10\}$ and determined  by the validation set for optimal balanced performance. 
For graph classification, the hidden dimension is set to $300$ to align with the pre-trained model. 
Then, the learning rate and the dropout rate are respectively set to $0.002$ and $0$. 
The hyper-parameter $r$ is selected from the range of $\{2,5,10\}$ and  determined  by the validation set for optimal balanced performance. 

\subsection{Comparison Results}
\textbf{Comparison baselines.}
For a comprehensive evaluation, we compare our proposed methods with the following baselines based on different prompting strategies: 
\begin{itemize}
    \item {Full Fine-tuning (FT):} updates all parameters of the pre-trained GNN backbone.
    \item {Node feature prompt methods:} inject learnable prompts  into node feature space, including GPF~\cite{gpf24} and its variant GPF-plus~\cite{gpf24}. For graph classification tasks, we also include GPPT~\cite{gppt22} and GraphPrompt~\cite{graphprompt23} for comparisons.
    \item {Edge feature prompt methods:} introduce trainable edge feature prompts, including EdgePrompt~\cite{EdgePrompt} and its variant {EdgePrompt+}~\cite{EdgePrompt}.
    \item {Edge weight prompt methods:} adjust the graph structure or edge weights, including GraphTOP~\cite{graphtop25} that performs topology-oriented prompting, and UniPrompt~\cite{uniprompt} that unifies graph adaptation via structure prompts.
    \item {Hybrid/Subgraph prompt methods:} integrate multiple graph prompting strategies or inject prompted subgraph, including VGP~\cite{vgp} that adopts hybrid prompt for vision-graph tasks and All in One~\cite{all_in_one23} that uses subgraph prompt. 
\end{itemize}

\textbf{Main results.}
Table~\ref{ncla} reports the average performance of all compared methods with five different data splits. 
Overall, LR-GMP consistently achieves the best performance on all datasets, demonstrating the effectiveness and robustness of our proposed approach. 
Specifically, we can have the following observations. First, compared with the traditional full fine-tuning (FT) method, LR-GMP can achieve higher performance. This validates that prompt tuning is more effective in alleviating the negative transfer problem and better preserving transferable knowledge from pre-trained models. 
Second, LR-GMP outperforms existing single graph prompt methods, including GPFs~\cite{gpf24}, EdgePrompts~\cite{EdgePrompt}, GraphTOP~\cite{graphtop25}, and UniPrompt~\cite{uniprompt}, which typically apply prompting to individual graph components. 
This demonstrates
the advantage of our proposed graph message prompt in enabling a more expressive learning space.
Third, LR-GMP consistently obtains better results than subgraph prompt methods (All in One~\cite{all_in_one23}). This further verifies that our proposed graph message prompting strategy can more effectively unify different GDPs and thus better guide fine-tuning for diverse graph tasks.
\begin{table*}[!htp]
\fontsize{9.3pt}{11pt}\selectfont
\renewcommand{\arraystretch}{1.2}
\centering
\caption{Comparison results of graph classification tasks on seven chemistry benchmarks, based on pre-trained GIN by EdgePred method. The best result is marked in bold.}
\begin{tabular}{l|ccccccc|c}
\toprule
{\begin{tabular}[c]{@{}l@{}}Method\end{tabular}} & BBBP & Tox21 & ToxCast & SIDER & ClinTox & MUV & BACE & \textbf{Average} \\
\midrule
FT & 66.56{\scriptsize$\pm$3.56} & 78.67{\scriptsize$\pm$0.35} & 66.29{\scriptsize$\pm$0.45} & 64.35{\scriptsize$\pm$0.78} & 69.07{\scriptsize$\pm$4.61} & 79.67{\scriptsize$\pm$1.70} & 80.90{\scriptsize$\pm$0.92} & 72.72 \\
GPPT & 64.13{\scriptsize$\pm$0.14} & 66.41{\scriptsize$\pm$0.04} & 60.34{\scriptsize$\pm$0.14} & 54.86{\scriptsize$\pm$0.25} & 59.81{\scriptsize$\pm$0.46} & 63.05{\scriptsize$\pm$0.34} & 70.85{\scriptsize$\pm$1.42} & 61.80 \\
GPPT w/o ol & 69.43{\scriptsize$\pm$0.18} & 78.91{\scriptsize$\pm$0.15} & 64.86{\scriptsize$\pm$0.11} & 60.94{\scriptsize$\pm$0.18} & 62.15{\scriptsize$\pm$0.69} & 82.06{\scriptsize$\pm$0.53} & 70.31{\scriptsize$\pm$0.99} & 70.97 \\
GraphPrompt & 69.29{\scriptsize$\pm$0.19} & 68.09{\scriptsize$\pm$0.19} & 60.54{\scriptsize$\pm$0.21} & 58.71{\scriptsize$\pm$0.13} & 55.37{\scriptsize$\pm$0.57} & 62.35{\scriptsize$\pm$0.44} & 67.70{\scriptsize$\pm$1.26} & 61.20 \\
GPF & 69.57{\scriptsize$\pm$0.21} & 79.74{\scriptsize$\pm$0.03} & 65.65{\scriptsize$\pm$0.30} & 67.20{\scriptsize$\pm$0.99} & 69.49{\scriptsize$\pm$5.17} & 82.86{\scriptsize$\pm$0.23} & 81.57{\scriptsize$\pm$1.08} & 74.51 \\
GPF-plus & 69.06{\scriptsize$\pm$0.68} & 80.04{\scriptsize$\pm$0.06} & 65.94{\scriptsize$\pm$0.31} & 67.51{\scriptsize$\pm$0.59} & 68.80{\scriptsize$\pm$2.58} & 83.13{\scriptsize$\pm$0.42}  & 81.75{\scriptsize$\pm$2.09} & 74.54 \\
VGP &  70.17{\scriptsize$\pm$0.40}  &  80.08{\scriptsize$\pm$0.04}  &  68.25{\scriptsize$\pm$0.16}  & 67.12{\scriptsize$\pm$0.63} &  74.82{\scriptsize$\pm$0.40}  &  82.64{\scriptsize$\pm$0.72} &  83.62{\scriptsize$\pm$0.31}   &  75.24  \\
\midrule
LR-GMP & \textbf{73.15{\scriptsize$\pm$0.40}}  &  \textbf{80.14{\scriptsize$\pm$0.22}} & \textbf{68.28{\scriptsize$\pm$0.45}}    &   \textbf{68.35{\scriptsize$\pm$0.38}}
 &\textbf{78.56{\scriptsize$\pm$2.12}} & \textbf{85.04{\scriptsize$\pm$0.43}}   &  \textbf{85.70{\scriptsize$\pm$1.28}} & \textbf{77.03}  \\
\bottomrule
\end{tabular}
\label{gcla}
\end{table*}

\textbf{Other results.}
Table~\ref{gcla}  reports the average performance of all methods under five random network
initializations on molecule graph classification tasks. 
Overall, LR-GMP achieves the best performance on all datasets, further demonstrating the effectiveness and generalizability of our proposed graph message prompting framework for graph-level tasks.
To be specific, we can find that LR-GMP consistently outperforms some representative simple graph prompt methods (GPPT~\cite{gppt22}, GraphPrompt~\cite{graphprompt23}, GPFs~\cite{gpf24}), which demonstrates the advantage of conducting prompting in graph message space. 
In addition, LR-GMP model obtains better results than VGP~\cite{vgp} that adopts hybrid graph prompt learning. This further illustrates that our graph message prompting framework provides a general and expressive paradigm to unify different GDPs, leading to better generalization performance.
%

\begin{figure}[!ht]
    \centering
    \includegraphics[width=0.485\textwidth]{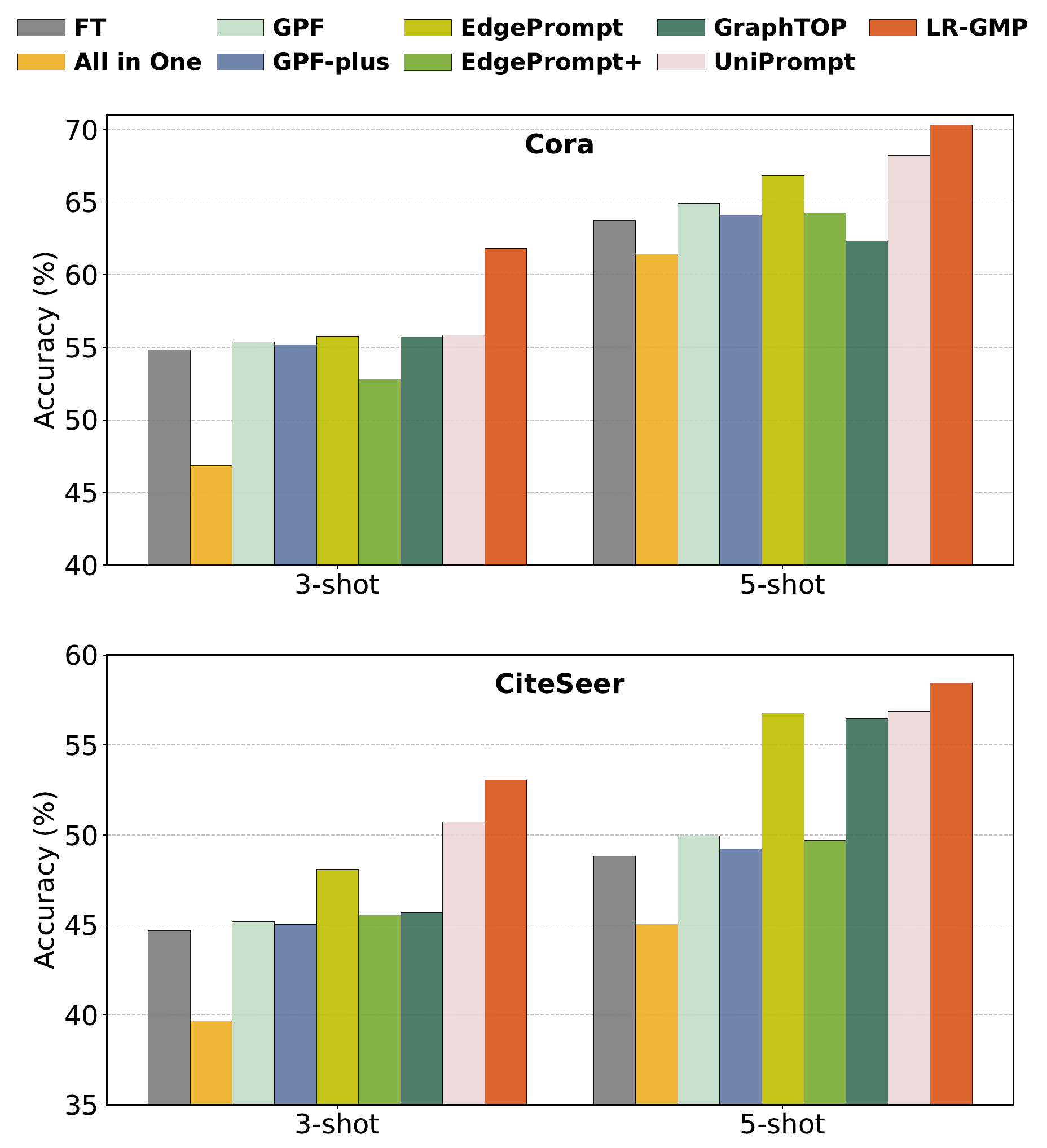}
    \caption{Comparison results of all methods under different shot settings.}
    \label{fig:shot}
\end{figure}
\subsection{Model Analysis}
\textbf{Impact of shot numbers.} In this subsection, we first investigate the performance of our proposed LR-GMP under different few-shot settings.  Fig.~\ref{fig:shot} illustrates the comparison results of various methods on Cora and CiteSeer datasets with $3$-shot and $5$-shot settings. We can observe that as the number of shots increases, the performance of all methods consistently improves, indicating that additional supervision benefits model prompt tuning. More importantly, LR-GMP achieves the best performance in all settings, demonstrating the effectiveness of our proposed message prompt in different few-shot settings. Finally, LR-GMP outperforms both simple graph prompts and hybrid paradigms across all settings. These results demonstrate that conducting prompting in the message space can enable more unified and informative message-level adaptation and thus lead to better generalization ability of scarce labeled data.

\textbf{Robustness to noisy graphs.} To further evaluate the robustness of our LR-GMP method under different graph perturbations, we adopt two representative attack strategies, namely  Random~\cite{li2020deeprobust} and Nettack~\cite{zugner2018adversarial,li2020deeprobust}, to generate noisy graph data. To be specific, 
the Random method performs a global attack by randomly flipping a proportion $p$ of edges in the graph and thus uniformly corrupts the overall topology. The Nettack method conducts a targeted attack to  perturb a small number of edges associated with specific nodes, where the perturbation number is controlled by $p$.
In this subsection, we use the noisy graphs and data splits published by previous work ProGNN~\cite{prognn}. 
Fig.~\ref{fig:robust_demo} illustrates the comparison results on Cora and CiteSeer datasets under different attack settings. Here, we can observe that (1) LR-GMP obviously achieves higher results than node-level prompt methods, demonstrating that our graph message prompting is more robust to structural corruption than approaches operating purely at the node level. (2) LR-GMP can obtain better or comparable performance than edge-level prompt methods in almost all cases, demonstrating that our graph message prompting can enable robust message passing in noisy graphs. Overall, these results show that our unified GMP can provide a more robust learning mechanism by conducting graph prompting in message space.  
\begin{figure*}[!htp]
        \centering
        \includegraphics[width=1\textwidth]{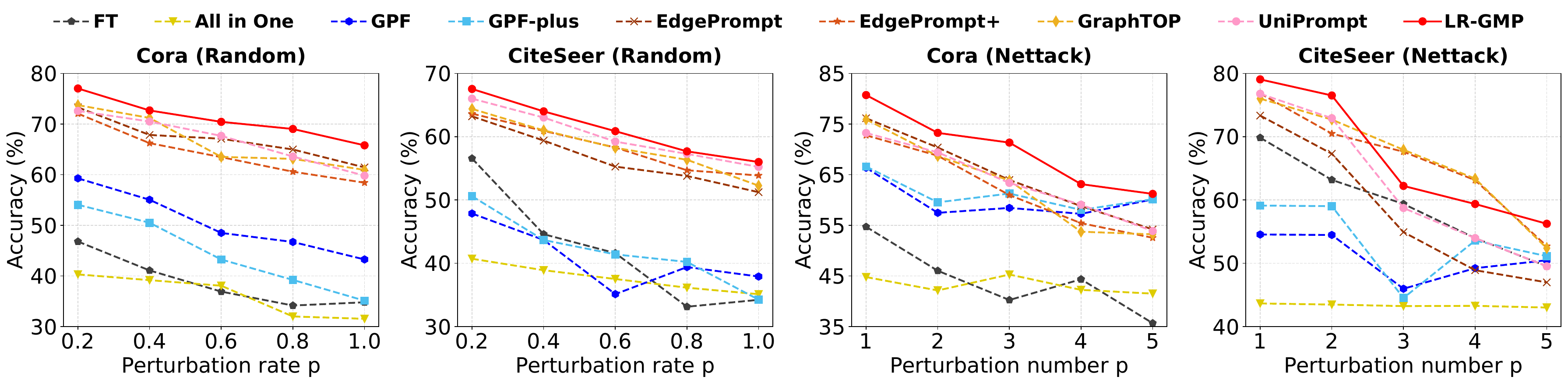}
    \caption{Performance comparison of different methods on noisy graph data, where $p$ denotes the global perturbation rate in Random method or the targeted perturbation number in Nettack method. }\label{fig:robust_demo}
\end{figure*}


\textbf{Model scalability.}
The proposed LR-GMP framework provides a general prompting mechanism that scales well across different GNN models. Since some compared methods rely on edge-weight prompting~\cite{graphtop25,uniprompt}, we adopt two GCN-based models, namely APPNP~\cite{APPNP} and GCNII~\cite{GCNII}, as backbones to ensure compatibility with these approaches.  
As shown in Table~\ref{backbones}, our proposed LR-GMP consistently achieves superior and comparable performance across different GNN backbones. This demonstrates the scalability and effectiveness of our message-level prompt mechanism, which serves as a unified prompting framework for different GNN models.
%
\begin{table}[!htp]
	\fontsize{7.6pt}{8pt}\selectfont
	\centering
    \renewcommand{\arraystretch}{1.2}
	\caption{Node classification results based on pre-trained APPNP and GCNII backbones by Edgepred-Gprompt pre-training method. The best result is marked in bold.}
	\begin{tabular}{l  | llll}
		\toprule	
        APPNP  & Cora & CiteSeer & PubMed & Photo\\
        \midrule
         FT  &  32.88{\scriptsize$\pm$10.96}  &  27.48{\scriptsize$\pm$6.31}  & 54.93{\scriptsize$\pm$7.64}  & 57.43{\scriptsize$\pm$9.92}  \\
		All in One  &  44.85{\scriptsize$\pm$5.67}  &  30.59{\scriptsize$\pm$2.57}  & 55.77{\scriptsize$\pm$7.40}  & 52.62{\scriptsize$\pm$8.42}  \\
		 GPF  &  47.16{\scriptsize$\pm$8.36}  &  30.66{\scriptsize$\pm$5.05}  & 56.35{\scriptsize$\pm$6.25}  & 50.74{\scriptsize$\pm$8.06}  \\
		 GPF-plus  &  47.21{\scriptsize$\pm$6.05}  &  30.21{\scriptsize$\pm$6.50}  & 55.80{\scriptsize$\pm$6.14}  & 55.02{\scriptsize$\pm$6.62}   \\
		 EdgePrompt  &  {47.36{\scriptsize$\pm$6.94}}  &  31.54{\scriptsize$\pm$8.61}  & {60.19{\scriptsize$\pm$5.32}}  & 54.08{\scriptsize$\pm$8.07}  \\
		 EdgePrompt+  &  43.96{\scriptsize$\pm$7.70}  &  28.44{\scriptsize$\pm$6.98}  & 51.71{\scriptsize$\pm$7.83}  & 52.83{\scriptsize$\pm$7.46}   \\
		 GraphTOP  &  44.53{\scriptsize$\pm$6.36}  &  34.50{\scriptsize$\pm$2.87}  & 55.42{\scriptsize$\pm$6.03}  & 54.27{\scriptsize$\pm$5.92}   \\
		UniPrompt  &  42.58{\scriptsize$\pm$7.89}  &  \textbf{45.10{\scriptsize$\pm$3.75}}  & 57.44{\scriptsize$\pm$6.96}  & {63.98{\scriptsize$\pm$9.12}}  \\
		LR-GMP &  \textbf{49.40{\scriptsize$\pm$5.52}}  &  {43.63{\scriptsize$\pm$5.93}}  & \textbf{62.75{\scriptsize$\pm$2.56}}  & \textbf{64.98{\scriptsize$\pm$3.20}}  \\
	    \midrule
        GCNII  & Cora & CiteSeer & PubMed & Photo\\
         \midrule
          FT  &  40.59{\scriptsize$\pm$11.49}  &  28.60{\scriptsize$\pm$7.73} & 53.00{\scriptsize$\pm$6.67}  & 51.18{\scriptsize$\pm$6.74}  \\
		 All in One  &  45.40{\scriptsize$\pm$6.82}  &  31.25{\scriptsize$\pm$3.02}  & 54.22{\scriptsize$\pm$6.29}  & 50.46{\scriptsize$\pm$7.20}   \\
		 GPF  &  {46.10{\scriptsize$\pm$8.19}}  &  28.90{\scriptsize$\pm$6.38}  & 55.94{\scriptsize$\pm$5.90}  & 55.24{\scriptsize$\pm$8.13}   \\
		 GPF-plus  &  45.87{\scriptsize$\pm$4.68}  &  28.24{\scriptsize$\pm$5.66}  & 55.85{\scriptsize$\pm$6.23}  & 55.68{\scriptsize$\pm$4.80}  \\
		 EdgePrompt  &  44.24{\scriptsize$\pm$10.17}  &  29.71{\scriptsize$\pm$8.26}  & 51.40{\scriptsize$\pm$10.87}  & 53.16{\scriptsize$\pm$6.01}  \\
		 EdgePrompt+  &  41.60{\scriptsize$\pm$4.70}  &  28.25{\scriptsize$\pm$5.38}  & {56.77{\scriptsize$\pm$4.77}}  & 53.56{\scriptsize$\pm$6.07}  \\
		 GraphTOP  &  43.45{\scriptsize$\pm$6.77}  &  31.78{\scriptsize$\pm$4.49}  & 56.60{\scriptsize$\pm$7.20}  & 58.47{\scriptsize$\pm$5.45}  \\
		 UniPrompt  &  39.20{\scriptsize$\pm$5.45}  &  {33.84{\scriptsize$\pm$7.84}}  & 54.96{\scriptsize$\pm$9.87}  & {58.67{\scriptsize$\pm$6.89}}  \\
		 LR-GMP & \textbf{47.56{\scriptsize$\pm$4.98}}  &  \textbf{39.32{\scriptsize$\pm$7.66}}  & \textbf{60.10{\scriptsize$\pm$3.60}}  & \textbf{61.86{\scriptsize$\pm$6.22}}   \\
        \bottomrule
	\end{tabular}
	\label{backbones}
\end{table}
\begin{figure}[!ht]
        \centering
        \includegraphics[width=0.485\textwidth]{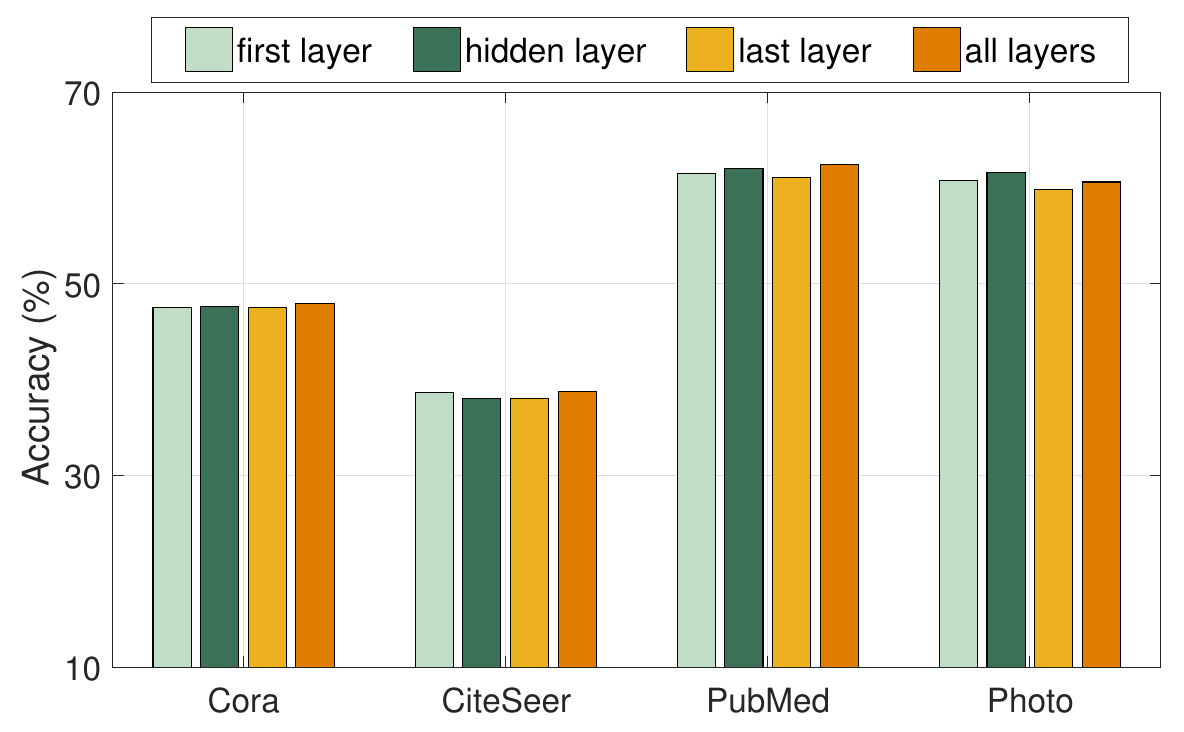}
    \caption{Performance comparison of LR-GMP with different prompt insertion layers including the first, hidden, last layer or across all layers.
    }\label{fig:prompt_layer}
\end{figure}

\textbf{Flexibility of prompt injection layers.} 
The proposed message prompting mechanism can be flexibly applied at different GNN layers to enable diverse layer-wise prompt designs. To analyze the flexibility of our LR-GMP regarding the prompt insertion layer, we evaluate the model's performance by injecting message prompts into the first layer, hidden layer, last layer, as well as across all layers. As illustrated in Fig.~\ref{fig:prompt_layer}, LR-GMP consistently achieves stable and competitive performance in all configurations. Note that, LR-GMP also obtains desirable results even when message prompt is applied to a single layer. 
These observations indicate that our proposed LR-GMP is inherently layer-agnostic and flexible to prompted layer selection, effectively modulating the layer-wise message passing for GNN learning.

\textbf{Parameter analysis.} As defined in Eq.(\ref{eq:LR-GMP}), the parameter $r$ is used to control the rank dimensions and thus determine the expressive capacity of LR-GMP.  In this subsection, we investigate the influence of different dimensions $r$ on the performance of LR-GMP method, as illustrated in Fig.~\ref{fig:k_r}.
Specifically, we can note that smaller rank $r$ can limit the learning capacity of the model while larger rank $r$ will introduce redundant parameters and increase the risk of overfitting, leading to performance degradation in most cases. Therefore, based on the observations in Fig.~\ref{fig:k_r}, we select $r$ from the range of $\{2,5,10\}$ to achieve a favorable trade-off between computational efficiency and desirable performance in our experiments.
\begin{figure}[!htp]
    \centering
    \includegraphics[width=0.24\textwidth]{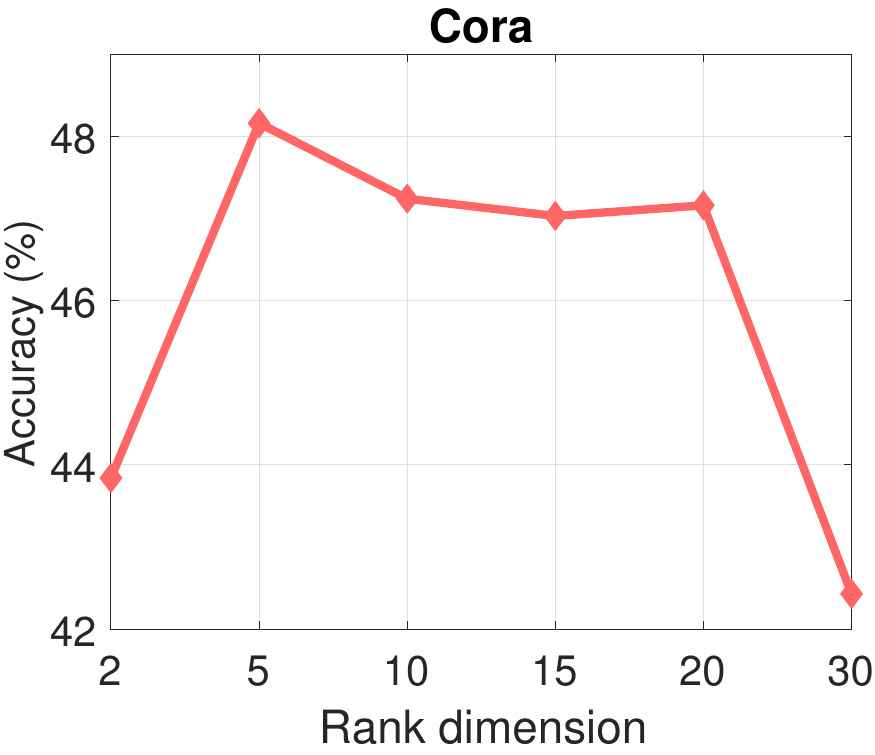}
    \includegraphics[width=0.24\textwidth]{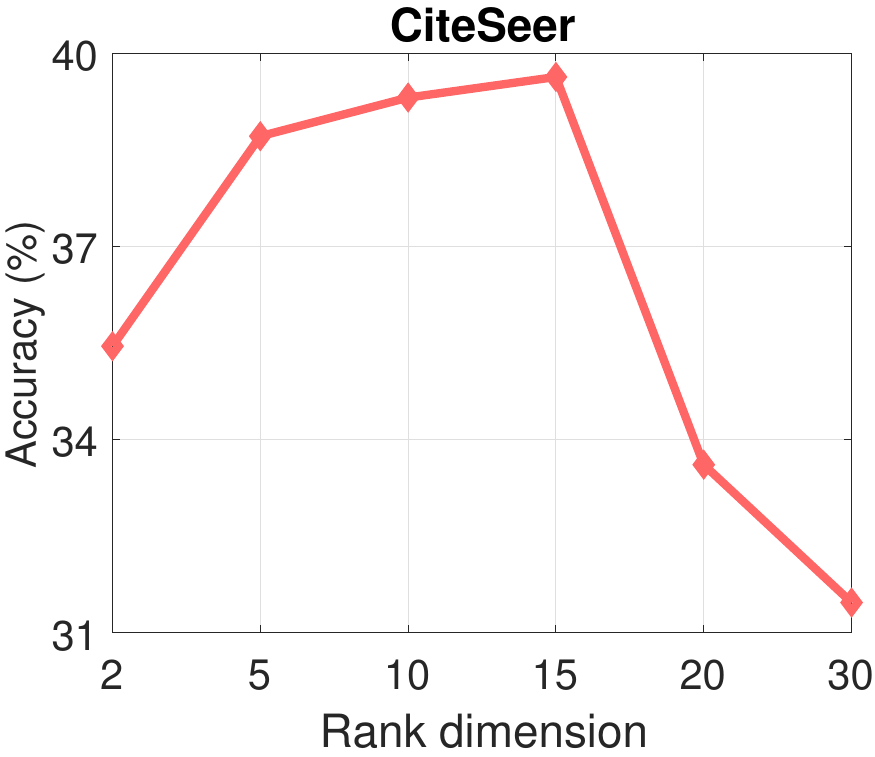}
     \\[6pt]
     \includegraphics[width=0.24\textwidth]{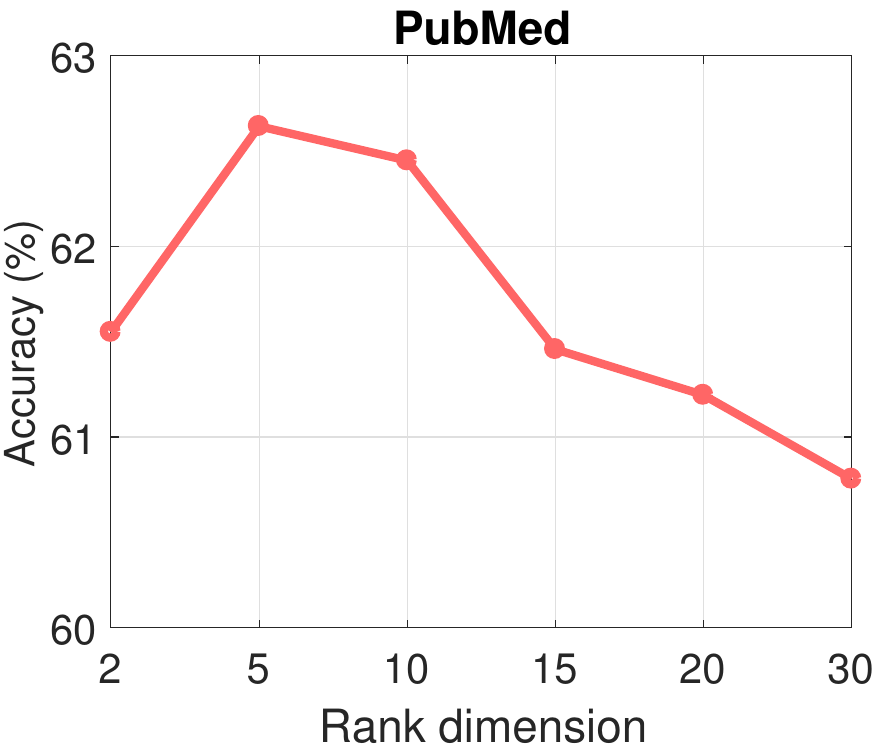}
    \includegraphics[width=0.24\textwidth]{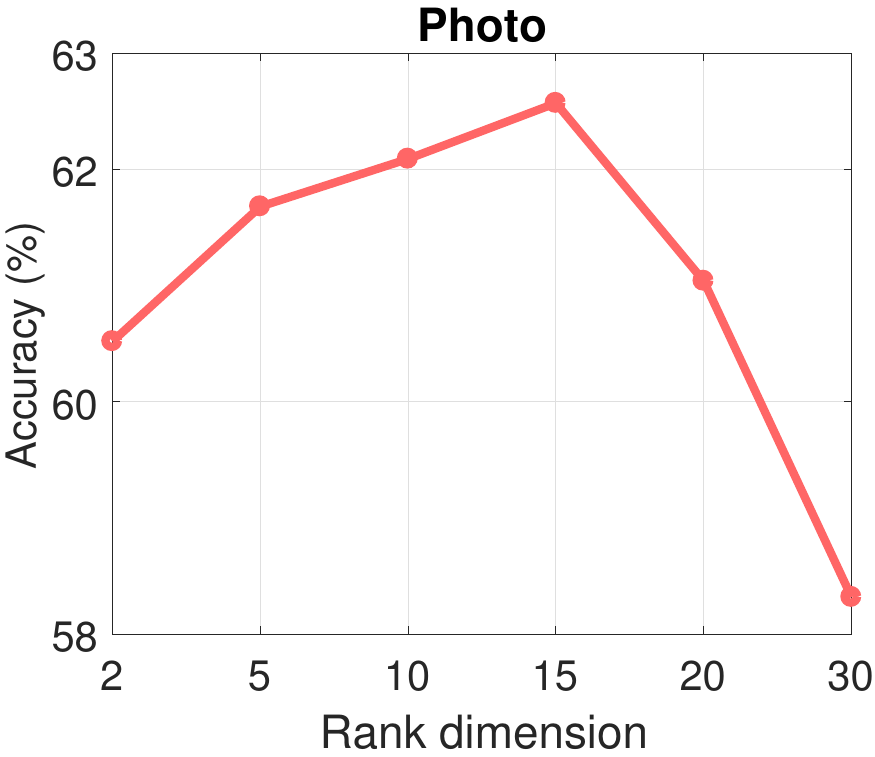}
    \caption{Results of LR-GMP with different rank dimensions $r$.}\label{fig:k_r}
\end{figure}

\section{Conclusions}
In this paper, we first revisit graph prompt learning and  reinterpret many existing graph data prompt models (GDPs) from Graph Message Prompt (GMP) perspective. We theoretically show that GMP scheme can unify many existing GDPs within a general message passing scheme. 
Based on this observation, we  then present a novel graph prompt learning approach, termed Low-Rank GMP (LR-GMP), to achieve a compact yet expressive prompt learning paradigm. 
Using LR-GMP, we can develop a general graph message prompt tuning framework for downstream task adaptation. Extensive experimental comparisons of various downstream tasks demonstrate the effectiveness, robustness and scalability of our proposed LR-GMP  on different graph learning datasets.  


\bibliography{ref}
\bibliographystyle{ieeetr}

\end{document}